\documentclass{llncs}

\usepackage{amsfonts}
\usepackage{amsmath}
\usepackage{latexsym}
\usepackage{color}
\usepackage{graphicx}
\usepackage{enumerate}
\usepackage{amstext}
\usepackage{epsfig}
\usepackage{comment}
\usepackage{verbatim}

\def\Pr{\mathbb{P}}
\def\dqed{\relax\tag*{\qed}}

\DeclareMathOperator*{\E}{\mathbb{E}}

\DeclareGraphicsExtensions{.pdf}

\newcommand{\e}{\epsilon}

\newcommand{\set}[1]{\{#1\}}

\newcommand{\ig}[2]{\includegraphics[scale=#1]{#2}}
\newcommand{\ignore}[1]{}

\ignore{
\newtheorem{theorem}{Theorem} 
\newtheorem{lemma}{Lemma} 
 
\newtheorem{corollary}{Corollary} 
\newtheorem{definition}{Definition} 
 
}
\newenvironment{proof*}{\trivlist
\item[\hskip\labelsep{\it\proofname}{.}]}

\title{Tight Lower Bound on the Probability of a Binomial Exceeding
  its Expectation}

\author{Spencer Greenberg \inst{1} \and
Mehryar Mohri \inst{1, 2}}

\institute{Courant Institute of Mathematical Sciences, \\
251 Mercer Street, New York, NY 10012. \and
Google Research, \\
76 Ninth Avenue, New York, NY 10011. }

\authorrunning{Greenberg and Mohri}

\begin{document}
\maketitle

\begin{abstract}

  We give the proof of a tight lower bound on the probability that a
  binomial random variable exceeds its expected value. The inequality
  plays an important role in a variety of contexts, including the
  analysis of relative deviation bounds in learning theory and generalization bounds for unbounded loss functions.

\end{abstract}

\section{Motivation}

This paper presents a tight lower bound on the probability that a
binomial random variable exceeds its expected value. If the binomial
distribution were symmetric around its mean, such a bound would be
trivially equal to $\frac{1}{2}$. And indeed, when the number of
trials $m$ for a binomial distribution is large, and the probability
$p$ of success on each trial is not too close to $0$ or to $1$, the
binomial distribution is approximately symmetric. With $p$ fixed,
and $m$ sufficiently large, the de Moivre-Laplace theorem tells us
that we can approximate the binomial distribution with a normal
distribution. But, when $p$ is close to $0$ or $1$, or the number of
trials $m$ is small, substantial asymmetry around the mean can arise.
Figure~\ref{fig:binomialExamples} illustrates this by showing the
binomial distribution for different values of $m$ and $p$.

The lower bound we prove has been invoked several times in the machine
learning literature, starting with work on relative deviation bounds
by Vapnik \cite{Vapnik98}, where it is stated without proof. Relative
deviation bounds are useful bounds in learning theory that provide
more insight than the standard generalization bounds because the
approximation error is scaled by the square root of the true error. In
particular, they lead to sharper bounds for empirical risk
minimization, and play a critical role in the analysis of
generalization bounds for unbounded loss functions
\cite{CortesMansourMohri2010}.

This binomial inequality is mentioned and used again without proof or
reference in \cite{AnthonyShawe-Taylor1993}, where the authors improve
the original work of \cite{Vapnik98} on relative deviation bounds by a
constant factor. The same claim later appears in \cite{Vapnik2006} and
implicitly in other publications referring to the relative deviations
bounds of Vapnik \cite{Vapnik98}.

To the best of our knowledge, there is no publication giving an actual
proof of this inequality in the machine learning literature.  Our
search efforts for a proof in the statistics literature were also
unsuccessful. Instead, some references suggest that such a
proof is indeed not available. In particular, we found one attempt to
prove this result in the context of the analysis of some
generalization bounds \cite{Jaeger2005}, but unfortunately the proof
is not sufficient to show the general case needed for the proof of
Vapnik \cite{Vapnik98}, and only pertains to cases where the number of
Bernoulli trials is `large enough'. No paper we have encountered contains a proof of the full theorem\footnote{After posting the preprint of our paper to
arxiv.org, we were contacted by the authors of \cite{Rigollet2011} who made
us aware that their paper \cite{Rigollet2011} contains a proof of a similar theorem. However,
their paper covers only the case where the bias $p$ of the binomial random variable satisfies
$p<\frac{1}{2}$, which is not
sufficiently general to cover the use of this theorem in the machine learning literature.}.
Our proof therefore seems to be the
first rigorous justification of this inequality in its full generality, which is needed for
the analysis of relative deviation bounds in machine learning.

In Section~\ref{sec:main}, we start with some preliminaries and then give the presentation of our main
result. In Section~\ref{sec:proofs}, we give a detailed proof of the inequality.

\begin{figure}[t]
\centering
\begin{tabular}{@{\hspace{-.2cm}}ccc}
\ig{.5}{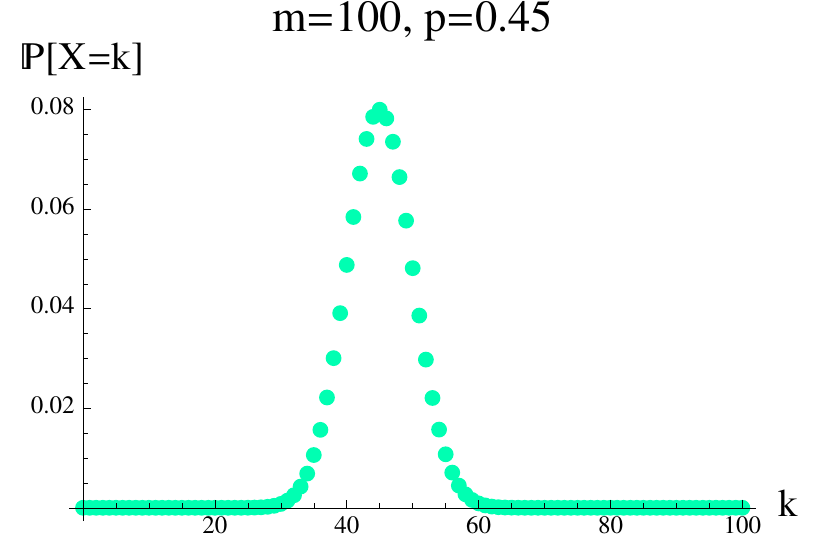} &
\ig{.5}{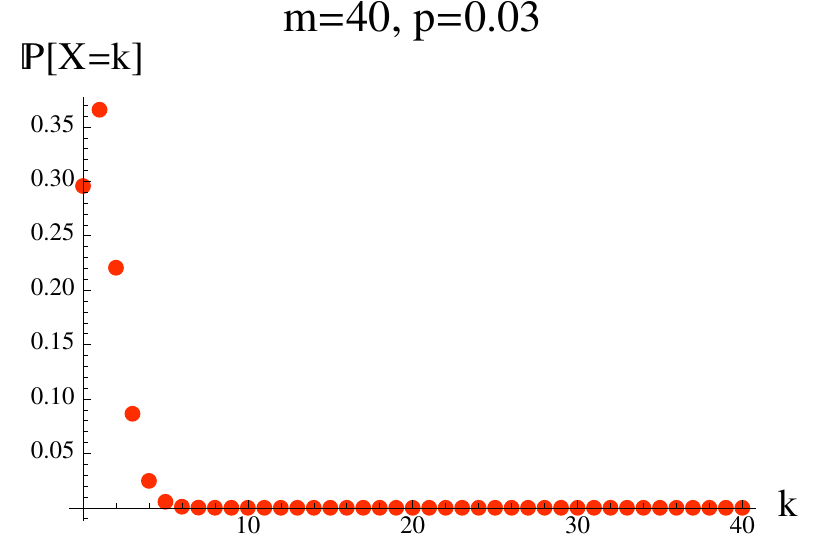} &
\ig{.5}{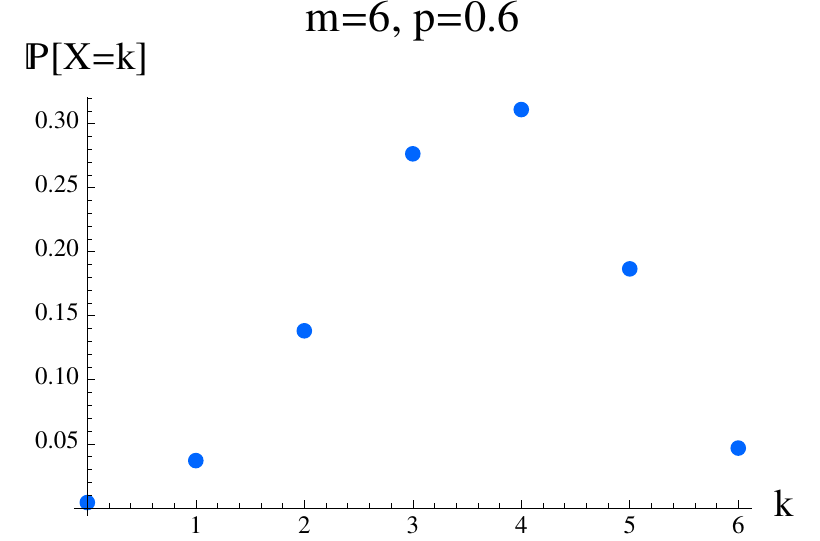}
\end{tabular}
\caption[]{Plots of the probability of getting different numbers of successes k, for the binomial distribution $B(m, p)$, shown for three
  different values of $m$, the number of trials, and $p$, the probability of a success on each trial. Note that in the second and third image, the distribution is clearly not symmetrical around its mean.}
\label{fig:binomialExamples}
\end{figure}

\section{Main result}
\label{sec:main}

The following is the standard definition of a binomial distribution.
\begin{definition}
  \emph{A random variable $X$ is said to be distributed according to the
  binomial distribution with parameters $m$ (the number of trials) and $p$ (the probability of success on each trial), if for $k = 0, 1, \ldots, m$ we have}
\begin{equation}
 \Pr[X = k] = \binom {m}{k} p^k (1 - p)^{m - k}.
\end{equation}
The binomial distribution with parameters $m$ and $p$ is denoted by
$B(m, p)$. It has mean $mp$ and variance $mp(1 - p)$.

\end{definition}

The following theorem is the main result of this paper.

\begin{theorem}
\label{th:main}
  For any positive integer $m$ and any probability p such that $p > \frac{1}{m}$, let $X$ be a random variable distributed according
  to $B(m, p)$.  Then, the following inequality holds:
\begin{equation}
\label{eq:main}
\Pr\big[X \geq  \E\left[X \right]\big] > \frac{1}{4}, 
\end{equation}
where $\E[X] = mp$ is the expected value of $X$. 
\end{theorem}
The lower bound is never reached but is approached asymptotically when
$m = 2$ as $p \to \frac{1}{2}$ from the right. Note that when $m=2$,
the case $p=\frac{1}{2}$ is excluded from consideration, due to our
assumption $p>\frac{1}{m}$. In words, the theorem says that a coin
that is flipped a fixed number of times always has a probability of
more than $ \frac{1}{4}$ of getting at least as many heads as the
expected value of the number of heads, as long as the coin's chance of
getting a head on each flip is not so low that the expected value is
less than or equal to 1. The inequality is tight, as illustrated by
Figure~\ref{fig:binomial}. In corollary~\ref{corollary:3} we prove a
bound on the probability of a binomial random variable being less than or equal to its expected value,
which is very similar to the bound here on such a random variable being greater than or equal to its expected value.

\begin{figure}[t]
\centering
\ig{.6}{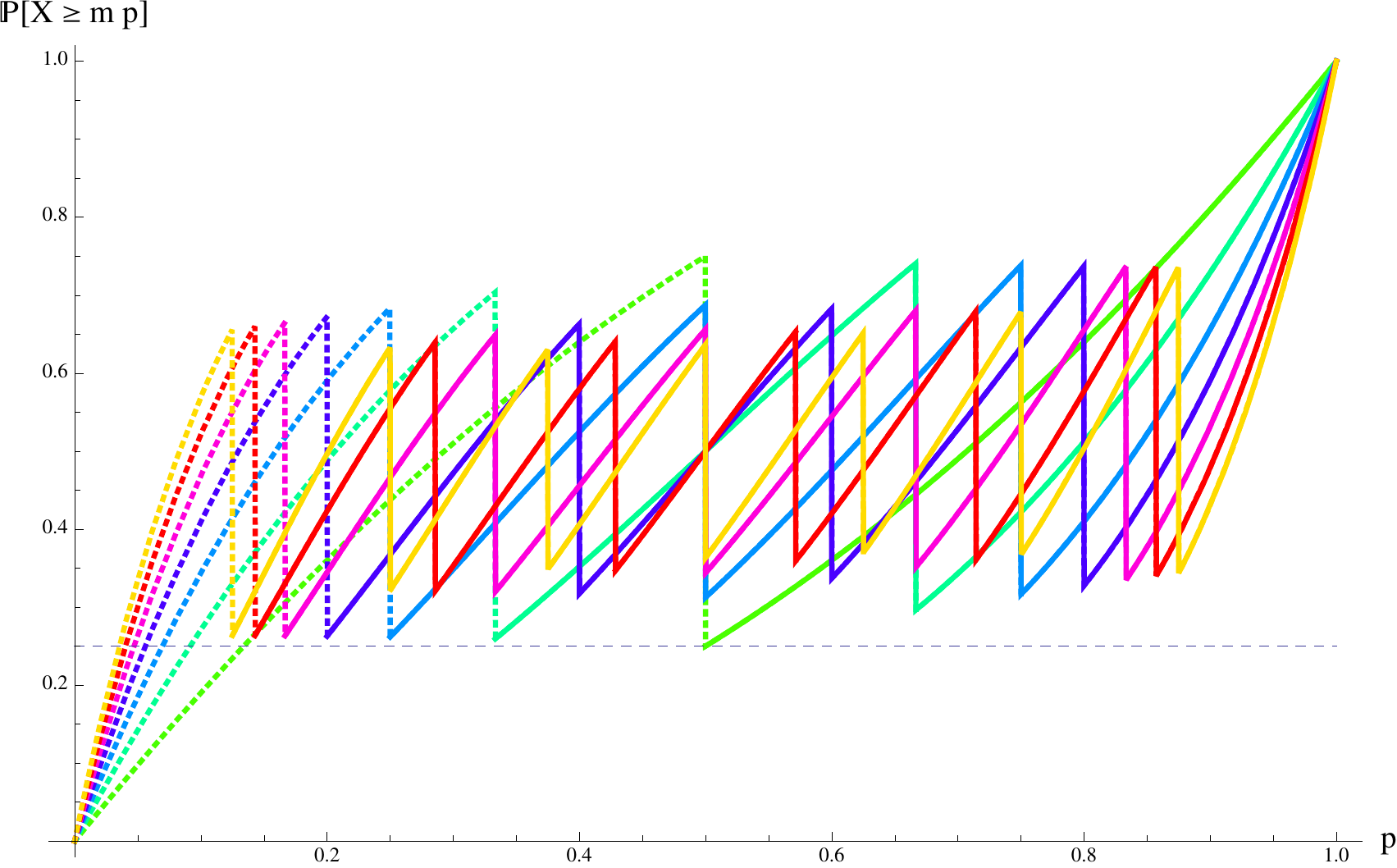}
\caption[]{This plot depicts $\Pr[X \geq \E[X]]$, the probability that
  a binomially distributed random variable $X$ exceeds its
  expectation, as a function of the trial success probability
  $p$. Each colored line corresponds to a different number of trials,
  $m = 2, 3, \ldots, 8$. Each colored line is dotted in the region
  where $p \le \frac{1}{m}$, and solid in the region that our proof
  pertains to, where $p >\frac{1}{m}$. The dashed horizontal line at
  $\frac{1}{4}$ represents the value of the lower bound. Our theorem
  is equivalent to saying that for all positive integers $m$ (not just
  the values of $m$ shown in the plot), the solid portions of the
  colored lines never cross below the dashed horizontal line. As can
  be seen from the figure, the lower bound is nearly met for many
  values of $m$.}
\label{fig:binomial}
\end{figure}

\section{Proof}
\label{sec:proofs}

Our proof of theorem~\ref{th:main} is based on the following series of lemmas and corollaries
and makes use of Camp-Paulson's normal approximation to the binomial
cumulative distribution function
\cite{JohnsonKotzBalakrishnan1995,JohnsonKempKotz2005,LeschJeske2009}.
We start with a lower bound that reduces the problem to a simpler one.

\begin{lemma}
\label{lemma:1}
  For all $k = 1, 2, \ldots , m - 1$ and $p\in
  (\frac{k}{m},\frac{k+1}{m}]$, the following inequality
  holds:
\begin{equation*}
\Pr_{X \sim B(m,p)} [X \geq \E[X]]\geq \Pr_{X \sim
    B(m,\frac{k}{m})}[X \geq  k+1]  .
\end{equation*}
\end{lemma}
\ignore{
 In other words, when the expected value of the number of heads for a
coin flipped $m$ times is in the range $(k,k+1]$, then the probability
that the number of actual heads is more than this expected value is
greater than or equal to the probability that a (different) coin with
bias $ \frac{k}{m}$ that is flipped $m$ times has $k+1$ or more heads.

\begin{figure}[t]
\centering
\ig{.6}{figures/binomial_gr6}
\end{figure}
}

\begin{proof*}
Let $X$ be a random variable distributed according to $B(m, p)$ and let
$F(m, p)$ denote $\Pr[X \geq \E[X]]$. Since $\E[X] = m p$, $F(m,p)$
can be written as the following sum:
\begin{equation*}
F(m,p) = \sum_{j = \lceil m p\rceil }^m \binom{m}{j}p^j(1 - p)^{m - j} .
\end{equation*}
We will consider the smallest value that $F(m,p)$ can take for $p\in
(\frac{1}{m},1]$ and $m$ a positive integer. Observe that if we
restrict $p$ to be in the half open interval $I_k = 
(\frac{k-1}{m},\frac{k}{m}]$, which represents a region
between the discontinuities of $F(m,p)$ which result from the factor $\lceil m p\rceil$, then we have $m p \in
(k-1,k]$ and so $\lceil m p\rceil = k$. Thus, we can write
\begin{equation*}
\forall p \in I_k , \hspace{1.7 mm} \forall k = 0, 1, \ldots , m - 1   \quad F(m, p) = \sum_{j = k}^m \binom{m}{j}p^j(1 - p)^{m - j}.    
\end{equation*}
The function $p \mapsto F(m, p)$ is differentiable for all $p\in
I_k$ and its differential is
\begin{equation*}
  \frac{\partial F(m,p)}{\partial p} = \sum_{j = k}^m
  \binom{m}{j}(1-p)^{m - j - 1}
  p^{j - 1} (j-m p).
\end{equation*}
Furthermore, for $p \in I_k$, we have $k \geq mp$, therefore $j \geq
mp$ (since in our sum $j \geq k$), and so $\frac{\partial
  F(m,p)}{\partial p} \geq 0$. The inequality is in fact strict when
$p \neq 0$ and $p \neq 1$ since the sum must have at least two terms
and at least one of these terms must be positive. Thus, the function
$p \mapsto F(m, p)$ is strictly increasing within each $I_k$. In view
of that, the value of $F(m, p)$ for $p \in I_{k+1}$ is lower bounded
by $\lim_{p\to \left(\frac{k}{m}\right)^+} F(m,p)$, which is given by
\begin{equation*}
\lim_{p\to \left(\frac{k}{m}\right)^+}   F(m,p) 
= \sum_{j=k+1}^m
\binom{m}{j}\left(\frac{k}{m}\right)^j\left(1-\frac{k}{m}\right)^{m-j} 
= \Pr_{X \sim B(m,\frac{k}{m})}[X \geq k+1].
\end{equation*}
Therefore, for $k=1, 2, \ldots, m-1$, whenever $p \in (\frac{k}{m},\frac{k+1}{m}]$ we have
\begin{equation*}
F(m,p) \geq  \Pr_{X \sim B(m,\frac{k}{m})}[X \geq k+1]. \dqed
\end{equation*}
\end{proof*}

\begin{corollary}
\label{corollary:1}
For all $p \in (\frac{1}{m}, 1)$, the following inequality holds:
\begin{equation*}
\Pr_{X \sim B(m,p)}[X \geq  \E[X]] \geq  1 - \max_{k \in
  \set{1, \ldots , m - 1}}\Pr_{X \sim B(m,\frac{k}{m})}[X \leq k]      .
\end{equation*}
\end{corollary}

\begin{proof}
By Lemma~\ref{lemma:1}, the following inequality holds
\begin{equation*}
 \Pr_{X \sim B(m,p)}[X \geq  \E[X]] 
 \geq \min_{k \in \set{1,..,m - 1}} \Pr_{X \sim
  B(m,\frac{k}{m})}[X \geq k + 1].
\end{equation*}
The right-hand side is equivalent to
\begin{equation*}
\min_{k \in \set{1,..,m - 1}} 1 - \Pr_{X \sim
  B(m,\frac{k}{m})}[X \leq  k] 
= 1 - \max_{k\in \set{1, \ldots , m -
    1}}\Pr_{X \sim B(m,\frac{k}{m})}[X \leq  k]
\end{equation*}
which concludes the proof.
\end{proof}

In view of Corollary~\ref{corollary:1}, in order to prove our main result it
suffices that we upper bound the expression
\begin{equation*}
\Pr_{X \sim B(m,\frac{k}{m})}[X \leq  k] 
= \sum_{j=0}^k \binom{m}{j}\left(\frac{k}{m}\right)^j\left(1-\frac{k}{m}\right)^{m-j}
\end{equation*}
by $\frac{3}{4}$ for all integers $m \geq 2$ and $1 \leq k \leq m -
1$. Note that the case $m = 1$ is irrelevant since the inequality $p >
\frac{1}{m}$ assumed for our main result cannot hold in that case, due
to $p$ being a probability.  The case $k = 0$ can also be ignored
since it corresponds to $p \le \frac{1}{m}$.  Finally, the case $k =
m$ is irrelevant, since it corresponds to $p > 1$. We note,
furthermore, that when $p = 1$ that immediately gives $\Pr_{X \sim
  B(m,p)}[X \geq \E[X]] = 1 \geq \frac{1}{4}$.

Now, we introduce some lemmas which will be used to prove our main
result.

\begin{lemma}
\label{lemma:2}
The following inequality holds for all  $k =1, 2, \ldots , m - 1$:
\begin{equation*}
\Pr_{X \sim B(m,\frac{k}{m})}[X \leq  k] \leq  \Phi \left[\frac{\beta_k \theta  +\frac{1}{3}\gamma_{m,k}}{\sqrt{ \beta_k+\gamma_{m,k}}}\right]+
\frac{0.007}{\sqrt{ 1-\frac{1}{m} }} ,
\end{equation*}
where $\Phi \colon x \mapsto \int_{-\infty }^x\frac{1}{\sqrt{2\pi }}
e^{\frac{-s^2}{2}} ds$ is the cumulative distribution function for the 
standard normal distribution and $\beta_k$, $\gamma_{m,k}$, and $\theta$
are defined as follows:
\begin{equation*}
\beta_k = \frac{1}{1+k}\left(1+\frac{1}{k}\right)^{2/3}\mspace{-25mu}, \quad 
\gamma_{m,k} = \frac{1}{m-k}, \quad
\theta = \frac{17}{3\ 2^{1/3}}-3\ 2^{1/3} \approx
 0.71787.
\end{equation*}
\end{lemma}

\ignore{
\begin{figure}[t]
\centering
\ig{.4}{figures/binomial_gr7}
\ig{.4}{figures/binomial_gr8}
\ig{.4}{figures/binomial_gr9}
\caption[]{}
\end{figure}
}

\ignore{
\begin{figure}[t]
\centering
\begin{tabular}{ccc}
\ig{.3}{figures/binomial_gr10} & 
\ig{.3}{figures/binomial_gr12} &
\ig{.3}{figures/binomial_gr13}\\
 (a) & (b) & (c)
\end{tabular}
\caption[]{Plots of (a) $\alpha_k$ and $\beta_k$; (b)
  $\alpha_k/\beta_k$; and (c) $\alpha_k/(\theta \beta_k)$ as a function of
  $k$.}
\label{fig:alpha}
\end{figure}

\begin{figure}[t]
\centering
\ig{.6}{figures/binomial_gr11}
\caption[]{Plot of $\e$ as a function of $k$.}
\label{fig:epsilon}
\end{figure}
}

\begin{figure}[t]
\centering
\ig{.6}{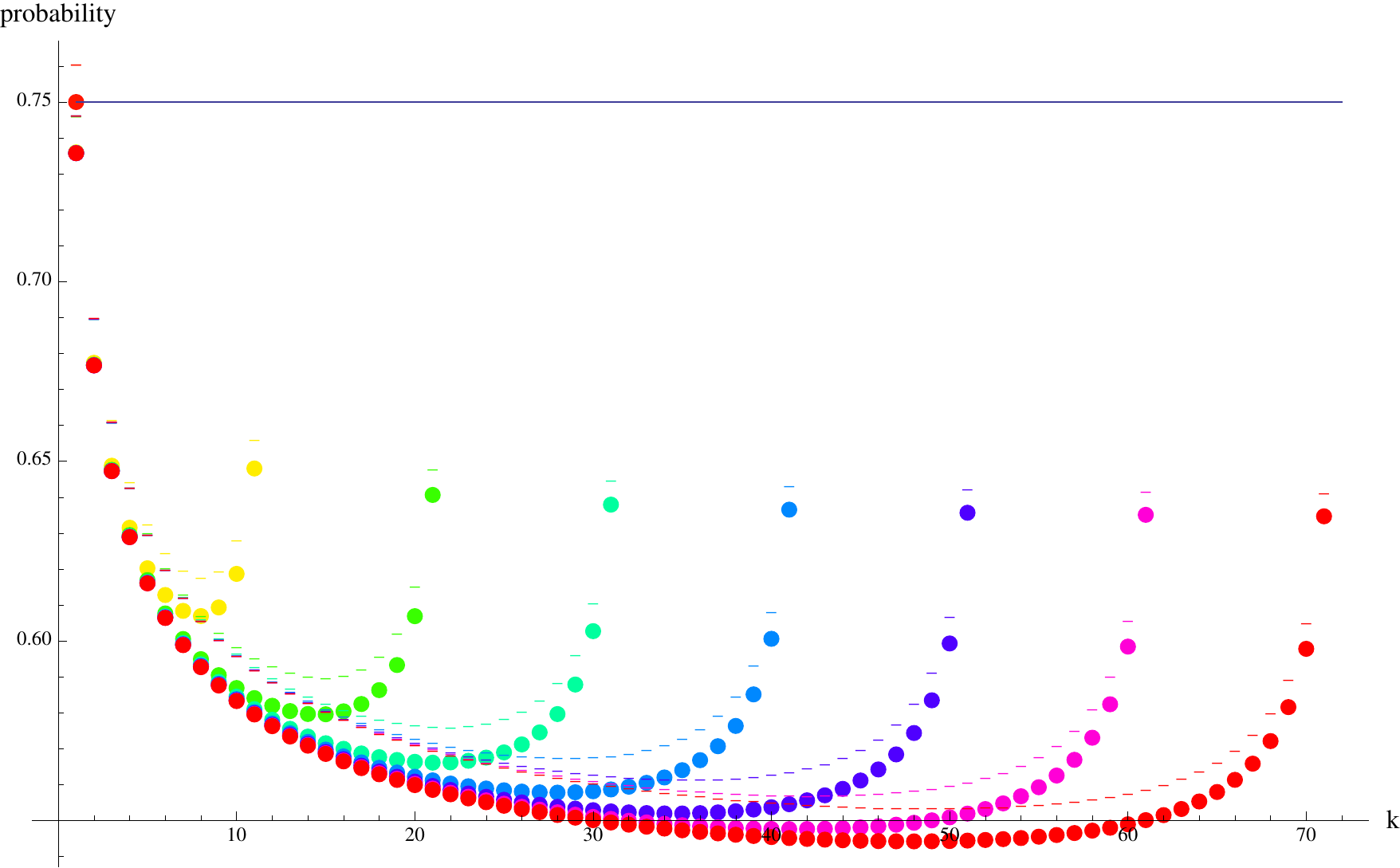}
\caption[]{For m=2, 22, $\ldots$, 72 and $1 \le k \le m-1$, this plot depicts the values of $\Pr_{X \sim B(m,\frac{k}{m})}[X \leq  k]$ as colored dots (with one color per choice of m), against the values of the bound from Lemma~\ref{lemma:2}, which are shown as short horizontal lines of matching color. The upper bound that we need to demonstrate, $\frac{3}{4}$, is shown as a blue horizontal line.}
\label{fig:binomial2}
\end{figure}

\begin{proof}
  Our proof makes use of Camp-Paulson's normal approximation to the
  binomial cumulative distribution function
  \cite{JohnsonKotzBalakrishnan1995,JohnsonKempKotz2005,LeschJeske2009},
  which helps us reformulate the bound sought in terms of the normal
  distribution.  The Camp-Paulson approximation improves on the
  classical normal approximation by using a non-linear transformation. This is useful for modeling the asymmetry that can occur in the binomial distribution. The Camp-Paulson bound can be stated as follows
  \cite{JohnsonKotzBalakrishnan1995,JohnsonKempKotz2005}:
\begin{equation*}
\left| \Pr_{X \sim B(m,p)}[X \leq j]- \Phi \Big[\frac{c-\mu }{\sigma }\Big] \right| \leq \frac{0.007}{\sqrt{m p (1-p)}}
\end{equation*}
where 
\begin{align*}
& c = (1-b)r^{1/3}, \quad \mu  = 1- a, \quad \sigma  = \sqrt{b r^{2/3}+a},\\
& a = \frac{1}{9 m - 9 j}, \quad b = \frac{1}{9j+9}, \quad r = \frac{(j+1)(1-p)}{m p - j p}.
\end{align*}
Plugging in the definitions for all of these variables yields
\begin{equation*}
\Phi \left[\frac{c-\mu }{\sigma }\right] = \Phi \left[\frac{\left(1-\frac{1}{9}\frac{1}{j+1}\right)\left(\frac{1}{p}\frac{(j+1)(1-p)}{m - j}\right)^{1/3}
+ \frac{1}{9}\frac{1}{m - j} -1}{\sqrt{\frac{1}{9}\frac{1}{j+1} \left(\frac{1}{p}\frac{(j+1)(1-p)}{m - j}\right)^{2/3}+\frac{1}{9}\frac{1}{m - j}}}\right].
\end{equation*}
Applying this bound to the case of interest for us where $p =
\frac{k}{m}$ and $j = k$, yields
\begin{equation*}
\frac{c-\mu }{\sigma }  = \frac{\alpha_k+\frac{1}{3}\gamma_{m,k}}{ \sqrt{\beta_k+\gamma_{m,k}}},
\end{equation*}
with $\alpha_k = \left(1+\frac{1}{k}\right)^{1/3}\left(3
  -\frac{1}{3}\frac{1}{1+k}\right)-3$, and with $\beta_k =   \frac{1}{1 + k}(1+\frac{1}{k})^{2/3}$. Thus, we can write 
\begin{equation}
\label{eq:cp}
\Pr_{X \sim B(m,\frac{k}{m})}[X \leq k]  \leq   \Phi \left[\frac{\alpha_k+\frac{1}{3}\gamma_{m,k}}{ \sqrt{\beta_k+\gamma
_{m,k}}}\right]  + \frac{0.007}{\sqrt{k \left(1-\frac{k}{m}\right)}}.
\end{equation}
To simplify this expression, we will first upper bound $\alpha_k$ in
terms of $\beta_k$. To do so, we consider the ratio
\begin{equation*}
\frac{\alpha_k}{\beta_k} 
= \frac{(1+\frac{1}{k})^{1/3}(3
    -\frac{1}{3}\frac{1}{1+k})-3}{
  \frac{1}{1 + k}(1+\frac{1}{k})^{2/3}}
= \frac{3 [(1 + \frac{1}{k})^{1/3} - 1]}{\frac{1}{1+k}(1+\frac{1}{k})^{2/3}} - \frac{1}{3 (1+\frac{1}{k})^{1/3}}.
\end{equation*}
Let $\lambda = (1+\frac{1}{k})^{1/3}$, which we can rearrange to write $\frac{1}{1 +
  k} = \frac{\lambda^3 - 1}{\lambda^3}$, with $\lambda \in (1,
2^{1/3}]$. Then, the ratio can be rewritten as follows:
\begin{equation*}
\frac{\alpha_k}{\beta_k} 
= \frac{3 \lambda^3 [(\lambda - 1]}{(\lambda^3 - 1) \lambda^2} -
\frac{1}{3 \lambda} = \frac{3 \lambda
}{1+\lambda +\lambda ^2} - \frac{1}{3 \lambda }.
\end{equation*}
The expression is differentiable and its differential is given by
\begin{align*}
 \frac{d}{d\lambda }\frac{\alpha_k}{\beta_k} 
& = 
\frac{3(1 + \lambda + \lambda^2) - 3 \lambda (2 \lambda +
  1)}{(1 + \lambda + \lambda^2)^2} + \frac{1}{3 \lambda^2} \\
& = \frac{(1 - \lambda^2)}{(1 + \lambda + \lambda^2)^2} + \frac{1}{3
  \lambda^2}
= \frac{-8 (\lambda -1)^4 -30 (\lambda -1)^3 -30 (\lambda - 1)^2 + 9}
{3 \lambda ^2 \left(1+\lambda +\lambda ^2\right)^2}.
\end{align*}
For $\lambda \in (1, 2^{\frac{1}{3}}]$, $\lambda - 1 \leq
2^{\frac{1}{3}} - 1$ $\le 0.26$, thus, the following inequality holds:
\begin{equation*}
8 (\lambda -1)^4 + 30 (\lambda -1)^3 + 30 (\lambda - 1)^2 \leq 
8 (2^{\frac{1}{3}} -1)^4 + 30 (2^{\frac{1}{3}} -1)^3 + 30
(2^{\frac{1}{3}} - 1)^2 \approx  2.59
< 9.
\end{equation*}
Thus, the derivative is positive, so $\frac{\alpha_k}{\beta_k}$ is an increasing function of
$\lambda$ on the interval $(1, 2^{\frac{1}{3}}]$ and its maximum is 
reached for $\lambda = 2^{\frac{1}{3}}$. For that choice of $\lambda$,
the ratio can be written
\begin{equation*}
\frac{3 \ 2^{\frac{1}{3}}
}{1+ 2^{\frac{1}{3}} + 2^{\frac{2}{3}}} - \frac{1}{3 \
  2^{\frac{1}{3}} } =  \frac{17}{3\ 2^{\frac{1}{3}}} - 3\
2^{\frac{1}{3}} = \theta \approx  0.717874,
\end{equation*}
which upper bounds $\frac{\alpha_k}{\beta_k}$.
Since $\Phi [x]$ is a strictly increasing function, using \\$\alpha_k
\leq \theta \beta_k$ yields
\begin{equation}
\label{eq:theta}
\Phi \left[\frac{\alpha_k+\frac{1}{3}\gamma_{m,k}}{ \sqrt{\beta_k+\gamma_{m,k}}}\right] \leq  \Phi \left[\frac{\beta_k \theta  +\frac{1}{3}\gamma
_{m,k}}{\sqrt{ \beta_k + \gamma_{m,k}}}\right].
\end{equation}

We now bound the term $\frac{0.007}{\sqrt{k (1-\frac{k}{m})}}$. The
quadratic function $k \mapsto k (1 - \frac{k}{m})$ for $k = 1, 2,
\ldots, m - 1$, achieves its minimum at $k=1$, giving $k (1 -
\frac{k}{m}) \geq (1 - \frac{1}{m})$. Thus, in view of \eqref{eq:cp}
and \eqref{eq:theta}, we can write
\begin{equation*}
\Pr\left[B(m,\frac{k}{m})\leq k\right]  \leq   \Phi \left[\frac{\beta_k \theta  +\frac{1}{3}\gamma_{m,k}}{ \sqrt{\beta_k+\gamma
_{m,k}}}\right]  + \frac{0.007}{\sqrt{1 - \frac{1}{m}}}.
\end{equation*}
This concludes the proof.\qed
\end{proof}

\ignore{
\begin{figure}[t]
\centering
\ig{.6}{figures/binomial_gr14}
\ig{.6}{figures/binomial_gr15}
\ig{.6}{figures/binomial_gr16}
\caption[]{}
\end{figure}
}

\begin{lemma}
\label{lemma:3} 
Let $\beta_k = \frac{1}{1 + k}\left(1+\frac{1}{k}\right)^{2/3}$ and $\gamma_{m,k} = \frac{1}{m - k}$ for
$m > 1$ and \\$k =1, 2, \ldots, m - 1$. Then, the following inequality
holds for $\theta =\frac{17}{3\ 2^{\frac{1}{3}}}-3\ 2^{\frac{1}{3}} $:
\begin{equation*}
\Phi \left[\frac{\beta_k \theta  +\frac{1}{3}\gamma_{m,k}}{\sqrt{ \beta_k+\gamma_{m,k}}}\right]\leq \Phi \left[\frac{\beta_k \theta  +\frac{1}{3}}{\sqrt{
\beta_k+1}}\right].
\end{equation*}
\end{lemma}

\ignore{
\begin{figure}[t]
\centering
\ig{.6}{figures/binomial_gr17}
\ig{.6}{figures/binomial_gr18}
\end{figure}
}

\begin{proof}
  Since for $k = 1, 2, \ldots, m - 1$, we have $\frac{1}{m - 1}
  \leq \gamma_{m,k} \leq 1$, the following inequality holds:
\begin{equation*}
\frac{\beta_k \theta  + \frac{1}{3}\gamma_{m,k}}{\sqrt{
    \beta_k+\gamma_{m,k}}} 
\leq  \max_{\gamma \in [0,1]} \frac{\beta_k \theta  +\frac{1}{3}\gamma
}{\sqrt{ \beta_k+\gamma }}.
\end{equation*}
Since $\beta_k = \frac{1}{1+k}\left(1+\frac{1}{k}\right)^{2/3} > 0$,
the function $\phi\colon \gamma \mapsto \frac{\beta_k \theta
  +\frac{1}{3}\gamma }{\sqrt{ \beta_k+\gamma }} $ is continuously
differentiable for $\gamma \in [0,1]$. Its derivative is given by
$\phi'(\gamma) = \frac{\gamma +\beta_k (2 - 3 \theta )}{6
  \left(\beta_k+\gamma \right){}^{3/2}}$. Since \\$2 - 3 \theta
\approx -0.1536 < 0$, $\phi'(\gamma)$ is non-negative if and only if
$\gamma \geq \beta_k (3 \theta -2)$. Thus, $\phi(\gamma)$ is
decreasing for $\gamma < \beta_k (3 \theta -2)$ and increasing for
values of $\gamma$ larger than that threshold. That implies that the
shape of the graph of $\phi(\gamma)$ is such that the function's value
is maximized at the end points.  So $\max_{\gamma \in [0,1]}
\phi(\gamma) = $ $\max(\phi(0), \phi(1)) = \max(\sqrt{\beta_k} \theta,
\frac{\beta_k \theta + \frac{1}{3}}{\sqrt{\beta_k + 1}})$. The
inequality $\sqrt{\beta_k} \theta \leq \frac{\beta_k \theta +
  \frac{1}{3}}{\sqrt{\beta_k + 1}}$ holds if and only if $\beta_k
(\beta_k + 1) \theta^2 \leq (\beta_k \theta + \frac{1}{3})^2$, that is
if $\beta_k \leq \frac{1}{\theta (9 \theta - 6)} \approx 3.022$. But
since $\beta_k$ is a decreasing function of $k$, it has $\beta_1
\approx 0.7937$ as its upper bound, and so this necessary requirement
always holds. That means that the maximum value of $\phi(\gamma)$ for
$\gamma \in [0,1]$ occurs at $\gamma=1$, yielding the upper bound
$\frac{\beta_k \theta + \frac{1}{3}}{\sqrt{\beta_k + 1}}$, which
concludes the proof.
\end{proof}

\ignore{
\begin{figure}[t]
\centering
\ig{.6}{figures/binomial_gr19}
\ig{.6}{figures/binomial_gr20}
\ig{.6}{figures/binomial_gr21}
\end{figure}

\begin{figure}[t]
\centering
\ig{.6}{figures/binomial_gr22}
\end{figure}

\begin{figure}[t]
\centering
\ig{.6}{figures/binomial_gr23}
\ig{.6}{figures/binomial_gr24}
\ig{.6}{figures/binomial_gr25}
\end{figure}
}

\begin{corollary}
\label{corollary:2}
The following inequality holds for all $m \geq 2$ and $k \geq  2$:
\begin{equation*}
\Pr_{X \sim B(m,\frac{k}{m})}[X \leq  k] \leq 0.7152.
\end{equation*}
\end{corollary}
\begin{proof}
By Lemmas~\ref{lemma:2}-\ref{lemma:3}, we can write
\begin{equation*}
\Pr_{X \sim B(m,\frac{k}{m})}[X \leq  k] \leq  \Phi \left[\frac{\beta_k \theta  +\frac{1}{3}}{\sqrt{ \beta_k+1}}\right]+ \frac{0.007}{\sqrt{
1-\frac{1}{m} }} .
\end{equation*}
Furthermore $\beta_k = \frac{1}{1+k}\left(1+\frac{1}{k}\right)^{2/3}$
is a decreasing function of $k$. Therefore, for $k \geq 2$, it must
always be within the range $\beta_k\in \left[\lim_{k\to \infty }
  \beta_k, \beta_2\right]$ = $\left[0,\frac{1}{2^{2/3}
    3^{1/3}}\right]\approx [0 ,0.43679] $, which implies
\begin{equation*}
\frac{\beta_k \theta  +\frac{1}{3}}{\sqrt{ \beta_k+1}} 
\leq  \max_{k \geq 2} \frac{\beta_k \theta  +\frac{1}{3}}{\sqrt{
\beta_k+1}} 
\leq   \max_{\beta \in [0, \beta_2 ]} \frac{\beta  \theta  +
  \frac{1}{3}}{\sqrt{\beta + 1}}.
\end{equation*}
The derivative of the differentiable function $g\colon \beta \mapsto
\frac{\beta \theta +\frac{1}{3}}{\sqrt{ \beta +1}}$ is given by
$g'(\beta) = \frac{3 (\beta + 2) \theta - 1}{6 (\beta + 1)^{3/2}}$.
We have that $3 (\beta + 2) \theta - 1 \geq 6 \theta - 1
\geq 6 \times .717 - 1 > 0$, thus $g'(\beta) \geq 0$. Hence, the maximum
of $g(\beta)$ occurs at $\beta_2$, where $g(\beta_2)$ is slightly smaller than
$0.53968$. Thus, we can write
\begin{equation*}
\Phi \left[\frac{\beta_k \theta  +\frac{1}{3}}{\sqrt{ \beta
     _k+1}}\right] \leq  \Phi \left[\underset{\beta \in
    \left[0,\beta_2 \right]}{\max
} \frac{\beta  \theta  +\frac{1}{3}}{\sqrt{ \beta +1}}\right] \leq   \Phi [0.53968] < 0.7053.
\end{equation*}
Now, $m \mapsto \frac{0.007}{\sqrt{ 1-\frac{1}{m} }}$ is a decreasing
of function of $m$, thus, for $m\geq 2$ it is maximized at $m = 2$,
yielding $\frac{0.007}{\sqrt{ 1-\frac{1}{m} }} \leq 0.0099$. Hence,
the following holds:
\begin{equation*}
\Phi \left[\frac{\beta_k \theta  +\frac{1}{3}}{\sqrt{ \beta_k+1}}\right]+ \frac{0.007}{\sqrt{ 1-\frac{1}{m} }}  \leq  0.7053 + 0.0099 =
0.7152,
\end{equation*}
as required.\qed
\end{proof}

The case $k = 1$ is addressed by the following lemma.

\begin{lemma}
\label{lemma:4}
Let $X$ be a random variable distributed according to $B(m,
\frac{1}{m})$. Then, the following equality holds for any $m \geq 2$:
\begin{equation*}
\Pr\left[X \leq  1\right] \leq  \frac{3}{4}.
\end{equation*}
\end{lemma}

\ignore{
\begin{figure}[t]
\centering
\begin{tabular}{cc}
\ig{.33}{figures/binomial_gr26} & 
\ig{.33}{figures/binomial_gr27}\\
(a) & (b)
\end{tabular}
\caption[]{(a) Plot of $\rho$ as a function of $m$; (b) plots of $(2
  m-1) \log\left[\frac{1}{1-\frac{1}{m}}\right]$ and $2 +\frac{1}{6
m^2}-\frac{1}{3 m^3}$ as a function of $m$.}
\label{fig:rho}
\end{figure}
}

\begin{proof}
For $m \geq 2$, define the function $\rho$ by
\begin{equation*}
\rho (m) = \Pr\left[X \leq  1\right]
=  \sum_{j=0}^1 \binom{m}{j}\left(\frac{1}{m}\right)^j\left(1-\frac{1}{m}\right)^{m-j}\mspace{-10mu}
= \left(1-\frac{1}{m}\right)^m + \left(1-\frac{1}{m}\right)^{m-1}\mspace{-30mu}.
\end{equation*}
The value of the function for $m = 2$ is given by
$
\rho (2) =  \left(1-\frac{1}{2}\right)^2 + \left(1-\frac{1}{2}\right)^{2-1} = \frac{3}{4}.
$
Thus, to prove the result, it suffices to show that $\rho$ is
non-increasing for $m \geq 2$. The derivative of $\rho$ is given for
all $m \geq 2$ by
\begin{equation*}
\rho' (m) = (m - 1)^{m - 1} m^{-m} \left(2 + (2m - 1) \log\left[1 - \frac{1}{m}
  \right]\right).
\end{equation*}
Thus, for $m \geq 2$, $\rho' (m) \leq 0$ if and only if $2 + (2 m - 1)
\log\left[1 - \frac{1}{m} \right]\leq 0$.  Now, for $m \geq 2$, using
the first three terms of the expansion $- \log\left[1 - \frac{1}{m}
\right] = \sum_{k=1}^{\infty } \frac{1}{k} \frac{1}{m^k}$, we can
write
\begin{equation*}
- (2m - 1) \log\bigg[1 - \frac{1}{m}
\bigg] \geq  (2m - 1)
\left(\frac{1}{m} + \frac{1}{2m^2} + \frac{1}{3 m^3}\right) = 2 + \frac{1}{6
m^2} - \frac{1}{3m^3} \geq 2,
\end{equation*}
where the last inequality follows from $\frac{1}{6 m^2}-\frac{1}{3
  m^3} \geq 0$ for $m \geq 2$. This shows that $\rho'(m) \leq 0$
for all $m \geq 2$ and concludes the proof. \qed
\end{proof}

We now complete the proof of our main result, by combining the previous
lemmas and corollaries.

\begin{proof}[of Theorem~\ref{th:main}]
By Corollary~\ref{corollary:1}, we can write
\begin{align*}
& \Pr_{X \sim B(m,p)}[X \geq  \E[X]] \\
& \geq 1 - \max_{k \in
  \set{1, \ldots , m - 1}}\Pr_{X \sim B(m,\frac{k}{m})}[X \leq k]\\
& = 1 - \max \Big\{\Pr_{X \sim B(m,\frac{1}{m})}[X \leq  1],  
\max_{k \in \set{2, \ldots , m - 1}} \Pr_{X \sim B(m,\frac{k}{m})}[X \leq
 k] \Big\} \\
& \geq 
1- \max  \set{\frac{3}{4},  0.7152} = 1 - \frac{3}{4} = \frac{1}{4},
\end{align*}
where the last inequality holds by Corollary~\ref{corollary:2} and
Lemma~\ref{lemma:4}.\qed
\end{proof}

As a corollary, we can produce a bound on the probability that a binomial random variable is less than or equal to its expected value, instead of greater than or equal to its expected value.

\begin{corollary}
\label{corollary:3}
  For any positive integer $m$ and any probability p such that \\$p < 1 - \frac{1}{m}$, let $X$ be a random variable distributed according
  to $B(m, p)$.  Then, the following inequality holds:
\begin{equation}
\label{eq:main}
\Pr\big[X \leq  \E\left[X \right]\big] > \frac{1}{4}, 
\end{equation}
where $\E[X] = mp$ is the expected value of $X$. 
\end{corollary}

\begin{proof}
Let $G(m,p)$ be defined as
\begin{equation*}
G(m,p)  \equiv \Pr\big[X \leq  \E\left[X \right]\big]  =  \sum_{j = 0}^{\lfloor m p\rfloor} \binom{m}{j}p^j(1 - p)^{m - j}
\end{equation*}
and let $F(m,p)$ be defined as before as
\begin{equation*}
F(m,p)  \equiv \Pr \big[X \geq \E\left[X \right] \big]  = \sum_{j = \lceil m p\rceil }^m \binom{m}{j}p^j(1 - p)^{m - j}.
\end{equation*}
Then, we can write, for $q=1-p$,
\begin{align*}
G(m, p) &= G(m,1-q) 
\\& = \sum_{j = 0}^{\lfloor m (1-q)\rfloor}  \binom{m}{j}(1-q)^j q^{m - j}
\\& =  \sum_{t = m - \lfloor m (1-q)\rfloor}^{m} \binom{m}{m-t}(1-q)^{m-t} q^{t}  
\\& =  \sum_{t = \lceil m q \rceil }^{m} \binom{m}{t}(1-q)^{m-t} q^{t}  
\\& =  F(m, q) = F(m, 1-p) > \frac{1}{4}
\end{align*}
with the inequality at the end being an application of theorem~\ref{th:main}, which holds so long as $q > \frac{1}{m}$, or equivalently, so long as $p < 1-\frac{1}{m}$. \qed
\end{proof}

\section{Conclusion}

We presented a rigorous justification of an inequality needed for the
proof of relative deviations bounds in machine learning theory. To our knowledge,
no other complete proof of this theorem exists in the literature, despite its repeated use.

\section*{Acknowledgment}
We thank Luc Devroye for discussions about the topic of this work. 

\bibliographystyle{abbrv}
\bibliography{binomial}

\begin{thebibliography}{1}

\bibitem{AnthonyShawe-Taylor1993}
M.~Anthony and J.~Shawe-Taylor.
\newblock A result of {Vapnik} with applications.
\newblock {\em Discrete Applied Mathematics}, 47:207 -- 217, 1993.

\bibitem{CortesMansourMohri2010}
C.~Cortes, Y.~Mansour, and M.~Mohri.
\newblock Learning bounds for importance weighting.
\newblock In {\em NIPS}, Vancouver, Canada, 2010. MIT Press.

\bibitem{Jaeger2005}
S.~A. Jaeger.
\newblock Generalization bounds and complexities based on sparsity and
  clustering for convex combinations of functions from random classes.
\newblock {\em Journal of Machine Learning Research}, 6:307--340, 2005.

\bibitem{Rigollet2011}
P.~Rigollet and X.~Tong
\newblock Neyman-Pearson Classification, Convexity and Stochastic Constraints
\newblock {\em Journal of Machine Learning Research}, 12:2831--2855, 2011.

\bibitem{JohnsonKempKotz2005}
N.~Johnson, A.~Kemp, and S.~Kotz.
\newblock {\em Univariate Discrete Distributions}.
\newblock Number v. 3 in Wiley series in probability and Statistics. Wiley \&
  Sons, 2005.

\bibitem{JohnsonKotzBalakrishnan1995}
N.~Johnson, S.~Kotz, and N.~Balakrishnan.
\newblock {\em Continuous univariate distributions}.
\newblock Number v. 2 in Wiley series in probability and mathematical
  statistics: Applied probability and statistics. Wiley \& Sons, 1995.

\bibitem{LeschJeske2009}
S.~M. Lesch and D.~R. Jeske.
\newblock Some suggestions for teaching about normal approximations to poisson
  and binomial distribution functions.
\newblock {\em The American Statistician}, 63(3):274--277, 2009.

\bibitem{Vapnik98}
V.~N. Vapnik.
\newblock {\em Statistical Learning Theory}.
\newblock Wiley-Interscience, 1998.

\bibitem{Vapnik2006}
V.~N. Vapnik.
\newblock {\em Estimation of Dependences Based on Empirical Data}.
\newblock Springer-Verlag, 2006.

\end{thebibliography}
\end{document}